\documentclass{article}

\usepackage[accepted]{icml2015}
\usepackage{url}
\usepackage{amsmath,amssymb,amsthm}
\usepackage{color}
\usepackage{mathtools}
\usepackage{graphicx}
\usepackage{verbatim}
\usepackage{enumitem}
\usepackage{booktabs}
\usepackage{hyperref}
\usepackage{algpseudocode}
\usepackage{algorithm}
\usepackage{times}
\usepackage{sidecap}
\usepackage[capitalize]{cleveref}

\newtheorem{thm}{Theorem}
\newtheorem{lemma}{Lemma}
\newtheorem{defn}{Definition}

\newcommand{\textremovedforicml}[1]{}
\newcommand{\lrbrack}[1]{\left[#1\right]}
\newcommand{\lrbrace}[1]{\left\{#1\right\}}
\newcommand{\lrparen}[1]{\left(#1\right)}
\DeclarePairedDelimiter{\norm}{\|}{\|}
\DeclarePairedDelimiter{\abs}{|}{|}

\newcommand{\grad}{\nabla}
\newcommand{\evalat}[1]{\big\rvert_{#1}}
\newcommand{\E}{\mathbb{E}}
\newcommand{\Var}{\mathrm{Var}}

\newcommand{\Eb}[2]{\E_{#1}\left[#2\right]}

\newcommand{\normone}[1]{\norm{#1}_1}

\newcommand{\pith}{\pi_{\theta}}
\newcommand{\cS}{\mathcal{S}}
\newcommand{\cA}{\mathcal{A}}
\newcommand{\cP}{P}
\newcommand{\Real}{\mathbb{R}}
\newcommand{\Ppi}{P_{\pi}}
\newcommand{\rhopi}{\rho_{\pi}}
\newcommand{\Qpi}{Q_{\pi}}
\newcommand{\tilth}{\tilde{\theta}}
\newcommand{\kl}[2]{D_{\rm KL}(#1 \ \| \ #2)}

\newcommand{\tilpi}{\tilde{\pi}}
 
\newcommand{\gradth}{\grad_{\theta}}
\newcommand{\tv}[2]{D_{TV}(#1 \ \| \ #2)}
\newcommand{\half}{\frac{1}{2}}

\newcommand{\tilG}{\tilde{G}}
\newcommand{\dP}{\Delta}
\newcommand{\Ppitil}{P_{\tilpi}}
\newcommand{\defeq}{\vcentcolon=}
\newcommand{\Api}{A_{\pi}}
\newcommand{\Vpi}{V_{\pi}}
\newcommand{\thold}{\theta_{\mathrm{old}}}
\newcommand{\thnew}{\theta_{\mathrm{new}}}

\newcommand{\piold}{\pi_{\mathrm{old}}}
\newcommand{\pinew}{\pi_{\mathrm{new}}}
\newcommand{\given}{|}

\newcommand{\Atari}{Atari}
\newcommand{\delay}{l}

\DeclareMathOperator*{\argmax}{arg\,max}
\newcommand{\maxtv}{\ensuremath D^{\rm max}_{\rm TV}}

\newcommand{\meankl}[1]{{\ensuremath \overline D_{\rm KL}^{#1}}}
\newcommand{\maxkl}{{\ensuremath D^{\rm max}_{\rm KL}}}

\newcommand{\maximize}{\operatorname*{maximize}}

\newcommand{\isd}{q}
\newcommand{\Apithold}{A_{\thold}}
\newcommand{\Qpithold}{Q_{\thold}}

\newcommand{\vine}{\textit{vine}}
\newcommand{\singlepath}{\textit{single path}}

\newcommand{\net}{\operatorname{NeuralNet}}
\newcommand{\mean}{\operatorname{mean}}
\newcommand{\stdev}{\operatorname{stdev}}

\newcommand{\muth}{\mu_{\theta}}
\newcommand{\muthold}{\mu_{\mathrm{old}}}
\newcommand{\klofmu}{\operatorname{kl}}

\newcommand{\maxadv}{\max_{s,a}{\abs*{\Api(s,a)}}}

\usepackage[compact]{titlesec}  
\titlespacing{\section}{0pt}{0pt}{0pt}

\newcommand{\footnoteremovedforicml}[1]{}
\icmltitlerunning{Trust Region Policy Optimization}

\begin{document}
\twocolumn[
\icmltitle{Trust Region Policy Optimization}

\icmlauthor{John Schulman}{joschu@eecs.berkeley.edu}
\icmlauthor{Sergey Levine}{slevine@eecs.berkeley.edu}
\icmlauthor{Philipp Moritz}{pcmoritz@eecs.berkeley.edu}
\icmlauthor{Michael Jordan}{jordan@cs.berkeley.edu}
\icmlauthor{Pieter Abbeel}{pabbeel@cs.berkeley.edu}

\icmladdress{University of California, Berkeley, Department of Electrical Engineering and Computer Sciences}

\author{John Schulman, Sergey Levine, Philipp Moritz, Michael I. Jordan, Pieter Abbeel}

\vskip 0.15in
]

\begin{abstract}
We describe an iterative procedure for optimizing policies, with guaranteed monotonic improvement.
By making several approximations to the theoretically-justified procedure, we develop a practical algorithm, called Trust Region Policy Optimization (TRPO).
This algorithm is similar to natural policy gradient methods and is effective for optimizing large nonlinear policies such as neural networks.
Our experiments demonstrate its robust performance on a wide variety of tasks: learning simulated robotic swimming, hopping, and walking gaits; and playing Atari games using images of the screen as input.
Despite its approximations that deviate from the theory, TRPO tends to give monotonic improvement, with little tuning of hyperparameters.
\end{abstract}

\section{Introduction}

Most algorithms for policy optimization can be classified into three broad categories:
(1) policy iteration methods, which alternate between estimating the value function under the current policy and improving the policy \cite{bert}; (2) policy gradient methods, which use an estimator of the gradient of the expected return (total reward) obtained from sample trajectories \cite{ps-rlmsp-08} (and which, as we later discuss, have a close connection to policy iteration); and (3) derivative-free optimization methods, such as the cross-entropy method (CEM) and covariance matrix adaptation (CMA), which treat the return as a black box function to be optimized in terms of the policy parameters \cite{szita2006learning}.

General derivative-free stochastic optimization methods such as CEM and CMA are preferred on many problems, because they achieve good results while being simple to understand and implement.
For example, while Tetris is a classic benchmark problem for approximate dynamic programming (ADP) methods\footnoteremovedforicml{ADP methods: algorithms based on exact algorithms for solving finite MDPs, which include policy iteration, value iteration, and the linear programming formulation of the Bellman equations \cite{bertsekas2011approximate}.}, stochastic optimization methods are difficult to beat on this task \cite{ggs-adpfp-13}.
For continuous control problems, methods like CMA have been successful at learning control policies for challenging tasks like locomotion when provided with hand-engineered policy classes with low-dimensional parameterizations \cite{wampler2009optimal}.
The inability of ADP and gradient-based methods to consistently beat gradient-free random search is unsatisfying, since gradient-based optimization algorithms enjoy much better sample complexity guarantees than gradient-free methods \cite{nemirovski2005efficient}.\footnoteremovedforicml{
The oracle complexity model of optimization assumes that we have an oracle that we can query to obtain a noisy measurement of the function value or subgradient, and it asks how many queries are required to achieve a given optimization error.
With an oracle that gives us subgradients, there are upper bounds on complexity that do not depend on dimensionality $d$ of the underlying space \cite{nemirovski2005efficient}.
However, the complexity scales as $d^2$ if we are given single noisy function evaluations \cite{shamir2012complexity}  and $d$ if we are given pairs of evaluations with the same noise sample \cite{duchi2013optimal}.}
Continuous gradient-based optimization has been very successful at learning function approximators for supervised learning tasks with huge numbers of parameters, and extending their success to reinforcement learning would allow for efficient training of complex and powerful policies.

In this article, we first prove that minimizing a certain surrogate objective function guarantees policy improvement with non-trivial step sizes.
Then we make a series of approximations to the theoretically-justified algorithm, yielding a practical algorithm, which we call trust region policy optimization (TRPO).
We describe two variants of this algorithm: first, the \textit{single-path} method, which can be applied in the model-free setting; second, the \textit{vine} method, which requires the system to be restored to particular states, which is typically only possible in simulation.
These algorithms are scalable and can optimize nonlinear policies with tens of thousands of parameters, which have previously posed a major challenge for model-free policy search \cite{dnp-spsr-13}.
In our experiments, we show that the same TRPO methods can learn complex policies for swimming, hopping, and walking, as well as playing Atari games directly from raw images.

\section{Preliminaries} \label{sec:preliminaries}

Consider an infinite-horizon discounted Markov decision process (MDP), defined by the tuple $(\cS, \cA, \cP, r, \rho_0, \gamma)$,
where $\cS$ is a finite set of states, $\cA$ is a finite set of actions,  $\cP: \cS \times \cA \times \cS \rightarrow \Real$ is the transition probability distribution, $r: \cS \rightarrow \Real$ is the reward function, $\rho_0: \cS \rightarrow \Real$ is the distribution of the initial state $s_0$, and $\gamma \in (0,1)$ is the discount factor.

Let $\pi$ denote a stochastic policy $\pi: \cS \times \cA \rightarrow [0,1]$, and let $\eta(\pi)$ denote its expected discounted reward:
\begin{align*}
&\eta(\pi) = \Eb{s_0,a_0,\dots}{\sum_{t=0}^{\infty} \gamma^t r(s_t)}, \text{ where}\\
&s_0 \sim \rho_0(s_0),\ a_{t} \sim \pi(a_t \given  s_t), \ s_{t+1} \sim P(s_{t+1} \given  s_t, a_t).\nonumber
\end{align*}
We will use the following standard definitions of the state-action value function $\Qpi$, the value function $\Vpi$, and the advantage function $\Api$:
\begin{align*}
&\Qpi(s_t, a_t) = \Eb{s_{t+1},a_{t+1},\dots}{\sum_{\delay=0}^{\infty} \gamma^{\delay} r(s_{t+\delay})}, \nonumber\\
&\Vpi(s_t) =\ \Eb{a_t,s_{t+1},\dots}{\sum_{\delay=0}^{\infty} \gamma^{\delay} r(s_{t+\delay})},\\
&A_{\pi}(s,a) =\  Q_{\pi}(s,a) - V_{\pi}(s), \text{ where} \\
&\quad a_{t} \sim \pi(a_t \given  s_t), s_{t+1} \sim P(s_{t+1} \given  s_t, a_t) \text{\ for $t\ge 0$}. \nonumber
\end{align*}
The following useful identity expresses the expected return of another policy $\tilpi$ in terms of the advantage over $\pi$, accumulated over timesteps (see \citet{kakade2002approximately} or Appendix \ref{sec:couplingproof} for proof):
\begin{align}
&\eta(\tilpi) = \eta(\pi) +  \Eb{s_0, a_0, \dots \sim \tilpi}{\sum_{t=0}^{\infty} \gamma^t A_{\pi}(s_t, a_t)} \label{eq:advid}
\end{align}
where the notation $\Eb{s_0, a_0, \dots \sim \tilpi}{\dots}$ indicates that actions are sampled $a_t \sim \tilpi(\cdot \given s_t)$.
Let $\rho_{\pi}$ be the (unnormalized) discounted visitation frequencies
\begin{align}
\rho_{\pi}(s) \!=\! P(s_0=s) \!+\! \gamma P(s_1= s) \!+\! \gamma^2 P(s_2 = s) \!+\! \dots, \nonumber
\end{align}
where $s_0 \sim \rho_0$ and the actions are chosen according to $\pi$.
We can rewrite \Cref{eq:advid} with a sum over states instead of timesteps:
\begin{align}
\eta(\tilpi)
&= \eta(\pi) +  \sum_{t=0}^{\infty} \sum_s P(s_t = s \given \tilpi)\sum_a \tilpi(a \given s) \gamma^t A_{\pi}(s, a) \nonumber\\
&= \eta(\pi) +  \sum_s  \sum_{t=0}^{\infty} \gamma^t P(s_t = s \given \tilpi)\sum_a \tilpi(a \given s) A_{\pi}(s, a) \nonumber\\
&= \eta(\pi) + \sum_s \rho_{\tilpi}(s) \sum_a \tilpi(a \given  s) \Api(s,a).
\label{eq:etaexact}
\end{align}
This equation implies that any policy update $\pi \rightarrow \tilpi$ that has a nonnegative expected advantage at \textit{every} state $s$, i.e., $\sum_a \tilpi(a\given s) A_{\pi}(s,a) \ge 0$, is guaranteed to increase the policy performance $\eta$, or leave it constant in the case that the expected advantage is zero everywhere.
This implies the classic result that the update performed by exact policy iteration, which uses the deterministic policy $\tilpi(s) = \argmax_a \Api(s, a)$, improves the policy if there is at least one state-action pair with a positive advantage value and nonzero state visitation probability, otherwise the algorithm has converged to the optimal policy.
However, in the approximate setting, it will typically be unavoidable, due to estimation and approximation error\footnoteremovedforicml{
By estimation error, we are referring to the fact certain quantities (like $Q$-values) are estimated by sampling, and there some random error when a finite number of samples is used.
By approximation error, we are referring to the fact that approximate policies and value functions (rather than lookup tables) which can't exactly  represent the true value functions or optimal policies.
},
that there will be some states $s$ for which the expected advantage is negative, that is, $\sum_a \tilpi(a \given  s) A_{\pi}(s,a) < 0$.
The complex dependency of $\rho_{\tilde{\pi}}(s)$ on $\tilde{\pi}$ makes Equation~(\ref{eq:etaexact}) difficult to optimize directly. Instead, we introduce the following local approximation to $\eta$:
\begin{align}
L_{\pi}(\tilpi) = \eta(\pi) + \sum_s \rho_{\pi}(s) \sum_a \tilpi(a \given  s) \Api(s,a). \label{eq:adv}
\end{align}
Note that $L_{\pi}$ uses the visitation frequency $\rhopi$ rather than $\rho_{\tilpi}$, ignoring changes in state visitation density due to changes in the policy.
However, if we have a parameterized policy $\pith$, where $\pith(a\given s)$ is a differentiable function of the parameter vector $\theta$, then $L_{\pi}$ matches $\eta$ to first order (see  \citet{kakade2002approximately}).
That is, for any parameter value $\theta_0$,
\begin{align}
L_{\pi_{\theta_0}}(\pi_{\theta_0}) &= \eta(\pi_{\theta_0}), \nonumber \\
\gradth L_{\pi_{\theta_0}}(\pith)\evalat{\theta=\theta_0} &= \gradth \eta(\pith)\evalat{\theta=\theta_0}. \label{eq:gradlgradeta}
\end{align}
\Cref{eq:gradlgradeta} implies that a sufficiently small step $\pi_{\theta_0} \rightarrow \tilpi$ that improves $L_{\pi_{\thold}}$ will also improve $\eta$, but does not give us any guidance on how big of a step to take.

To address this issue, \citet{kakade2002approximately} proposed a policy updating scheme called conservative policy iteration, for which they could provide explicit lower bounds on the improvement of $\eta$.
To define the conservative policy iteration update, let $\piold$ denote the current policy, and let $\pi' = \argmax_{\pi'} L_{\piold}(\pi')$.
The new policy $\pinew$ was defined to be the following mixture:
\begin{align}
\pinew(a\given s) = (1-\alpha) \piold(a\given s) + \alpha \pi'(a\given s). \label{eq:cpi}
\end{align}
Kakade and Langford derived the following lower bound:
\begin{align}
\eta(\pinew)
&\!\ge\! L_{\piold}(\pinew) - \frac{2 \epsilon \gamma  }{(1-\gamma)^2} \alpha^2 \nonumber\\
&\text{ where } \epsilon = \max_s \abs*{ \Eb{a \sim \pi'(a\given s)} {\Api(s,a) }}. \label{eq:kakbound0}
\end{align}
(We have modified it to make it slightly weaker but simpler.)
Note, however, that so far this bound only applies to mixture policies generated by Equation~(\ref{eq:cpi}). This policy class is unwieldy and restrictive in practice, and it is desirable for a practical policy update scheme to be applicable to all general stochastic policy classes.

\section{Monotonic Improvement Guarantee for General Stochastic Policies} \label{sec:improvetheory}

\Cref{eq:kakbound0}, which applies to conservative policy iteration, implies that a policy update that improves the right-hand side is guaranteed to improve the true performance $\eta$.
Our principal theoretical result is that the policy improvement bound in \Cref{eq:kakbound0} can be extended to general stochastic policies, rather than just mixture polices, by replacing $\alpha$ with a distance measure between $\pi$ and $\tilpi$, and changing the constant $\epsilon$ appropriately. Since mixture policies are rarely used in practice, this result is crucial for extending the improvement guarantee to practical problems. The particular distance measure we use is the total variation divergence, which is defined by
$\tv{p}{q} = \half \sum_i \abs{p_i - q_i}$ for discrete probability distributions $p,q$.\footnote{Our result is straightforward to extend to continuous states and actions by replacing the sums with integrals.}
Define $\maxtv(\pi,\tilpi)$ as
\begin{align}
\maxtv(\pi,\tilpi) = \max_s \tv{\pi(\cdot \given  s)}{\tilpi(\cdot \given  s) }. \label{eq:maxtv}
\end{align}
\begin{thm}\label{thm:impthm}
Let $\alpha = \maxtv(\piold,\pinew)$.
Then the following bound holds:
\begin{align}
&\eta(\pinew) \ge L_{\piold}(\pinew) - \frac{4 \epsilon \gamma  }{(1-\gamma)^2} \alpha^2 \nonumber\\
&\quad\text{ where } \epsilon = \maxadv
\end{align}
\end{thm}
We provide two proofs in the appendix.
The first proof extends Kakade and Langford's result
using the fact that the random variables from two distributions with total variation divergence less than $\alpha$ can be coupled, so that they are equal with probability $1-\alpha$.
The second proof uses perturbation theory.

Next, we note the following relationship between the total variation divergence and the KL divergence (\citet{pollardbook}, Ch.
3):
$\tv{p}{q}^2 \le \kl{p}{q}$.
Let  $\maxkl(\pi,\tilpi) = \max_s \kl{\pi(\cdot \given  s)}{\tilpi(\cdot \given  s)}$. The following bound then follows directly from \Cref{thm:impthm}:
\begin{align}
&\eta(\tilpi)  \ge L_{\pi}(\tilpi) - C \maxkl(\pi,\tilpi), \nonumber\\
&\qquad \text{ where } C = \frac{4 \epsilon \gamma}{(1-\gamma)^2}. \label{eq:klminorizer}
\end{align}
\Cref{alg:pialg} describes an approximate policy iteration scheme based on the policy improvement bound in Equation~(\ref{eq:klminorizer}). Note that for now, we assume exact evaluation of the advantage values $\Api$.

\begin{algorithm}
\begin{algorithmic}
\State Initialize $\pi_0$.
\For{$i=0,1,2,\dots$ until convergence}
\State Compute all advantage values $A_{\pi_i}(s,a)$.
\State Solve the constrained optimization problem
\begin{align*}
\qquad\pi_{i+1} &=\argmax_{\pi} \lrbrack{ L_{\pi_i}(\pi) - C \maxkl(\pi_i,\pi) } \nonumber \\
&\!\!\!\!\!\text{ where $C=4\epsilon\gamma/(1-\gamma)^2$}\\
&\!\!\!\!\!\text{ and } L_{\pi_i}(\pi) \!=\! \eta(\pi_i) \!+\! \sum_s \rho_{\pi_i}\!(s)\! \sum_a \! \pi(a \given  s) A_{\pi_i}(s,a) \nonumber
\end{align*}
\EndFor
\end{algorithmic}
\caption{Policy iteration algorithm guaranteeing non-decreasing expected return $\eta$
\label{alg:pialg}}
\end{algorithm}
It follows from \Cref{eq:klminorizer} that \Cref{alg:pialg} is guaranteed to generate a monotonically improving sequence of policies
$\eta(\pi_0) \le \eta(\pi_1) \le \eta(\pi_2) \le \dots$.
To see this, let $M_i(\pi) = L_{\pi_i}(\pi) - C \maxkl(\pi_i, \pi)$. Then
\begin{align}
&\eta(\pi_{i+1})
\ge  M_i(\pi_{i+1}) \text{ by Equation~\eqref{eq:klminorizer}} \nonumber \\
&\eta(\pi_{i})
=  M_i(\pi_i), \text{ therefore,} \nonumber \\
&\eta(\pi_{i+1}) - \eta(\pi_i) \ge M_i(\pi_{i+1}) - M(\pi_i). \label{eq:mmalgs}
\end{align}
Thus, by maximizing $M_i$ at each iteration, we guarantee that the true objective $\eta$ is non-decreasing.
This algorithm is a type of minorization-maximization (MM) algorithm \cite{hunter2004tutorial}, which is a class of methods that also includes expectation maximization.
In the terminology of MM algorithms, $M_i$ is the surrogate function that minorizes $\eta$ with equality at $\pi_i$.
This algorithm is also reminiscent of proximal gradient methods and mirror descent.\footnoteremovedforicml{
The policy update algorithm we have described has precisely the form of a proximal gradient algorithm \cite{parikh2013proximal} (with a non-quadratic penalty function).
That is, we maximize an affine approximation to the objective (note that $L$ is affine in terms of the probabilities) plus a regularization term (KL divergence) that reduces the step length.
Relatedly, the policy update is the same update that is performed in online mirror descent \cite{shalev2011online}, using entropy as the regularizer.
The KL divergence arises as the Bregman divergence from the entropy regularizer; this form of online mirror descent is commonly used for optimizing over probability distributions \cite{kivinen1997exponentiated}.}

\textit{Trust region policy optimization}, which we propose in the following section, is an approximation to \Cref{alg:pialg}, which uses a constraint on the KL divergence rather than a penalty to robustly allow large updates.

\section{Optimization of Parameterized Policies} \label{sec:parameterized}

In the previous section, we considered the policy optimization problem independently of the parameterization of $\pi$ and under the assumption that the policy can be evaluated at all states. We now describe how to derive a practical algorithm from these theoretical foundations, under finite sample counts and arbitrary parameterizations.

Since we consider parameterized policies $\pith(a \given  s)$ with parameter vector $\theta$,
we will overload our previous notation to use functions of $\theta$ rather than $\pi$,
e.g. $\eta(\theta)  \defeq \eta(\pith)$, $L_{\theta}(\tilth) \defeq L_{\pith}(\pi_{\tilth})$, and $\kl{\theta}{\tilth} \defeq \kl{\pi_{\theta}}{\pi_{\tilth}}$.
We will use $\thold$ to denote the previous policy parameters that we want to improve upon.

The preceding section showed that $\eta(\theta) \ge L_{\thold}(\theta)-C \maxkl(\thold, \theta)$, with equality at $\theta=\thold$.
Thus, by performing the following maximization, we are guaranteed to improve the true objective $\eta$:
\begin{align*}
\maximize_{\theta} \lrbrack {L_{\thold}(\theta) - C \maxkl(\thold, \theta) }.
\end{align*}
In practice, if we used the penalty coefficient $C$ recommended by the theory above, the step sizes would be very small.
One way to take larger steps in a robust way is to use a constraint on the KL divergence between the new policy and the old policy, i.e., a trust region constraint\footnoteremovedforicml{For background on trust region methods in nonlinear optimization, see \cite{wright1999numerical}, chapter 4. A variety of nonlinear optimization algorithms work by solving a series of trust region subproblems, where a local approximation to the objective is optimized subject to a trust region constraint, which restricts the solution to the region where the approximation is valid.}:
\begin{align}
&\maximize_{\theta}  L_{\thold}(\theta) \label{eq:maxklconst} \\ &\text{\ \  subject to }  \maxkl(\thold,\theta) \le \delta. \nonumber
\end{align}
This problem imposes a constraint that the KL divergence is bounded at every point in the state space.
While it is motivated by the theory, this problem is impractical to solve due to the large number of constraints.
Instead, we can use a heuristic approximation which considers the average KL divergence:
\begin{align*}
\meankl{\rho}(\theta_1,\theta_2) \defeq \Eb{s \sim \rho}{\kl{\pi_{\theta_1}(\cdot \given  s)}{\pi_{\theta_2}(\cdot\given s)}}.
\end{align*}
We therefore propose solving the following optimization problem to generate a policy update:
\begin{align}
&\maximize_{\theta}  L_{\thold}(\theta) \label{eq:trprob}\\ &\text{\ \  subject to }  \meankl{\rho_{\thold}}(\thold,\theta) \le \delta. \nonumber
\end{align}
Similar policy updates have been proposed in prior work \cite{bagnell2003covariant,peters2008natural,pma-reps-10}, and we compare our approach to prior methods in Section~\ref{sec:related} and in the experiments in Section~\ref{sec:expt}. Our experiments also show that this type of constrained update has similar empirical performance to the maximum KL divergence constraint in Equation~(\ref{eq:maxklconst}).

\section{Sample-Based Estimation of the Objective and Constraint} \label{sec:samp}

The previous section proposed a constrained optimization problem on the policy parameters (Equation~\eqref{eq:trprob}),
which optimizes an estimate of the expected total reward $\eta$ subject to a constraint on the change in the policy at each update.
This section describes how the objective and constraint functions can be approximated using Monte Carlo simulation.

We seek to solve the following optimization problem, obtained by expanding
$L_{\thold}$ in Equation~\eqref{eq:trprob}:
\begin{align}
\maximize_{\theta}
\sum_s &\rho_{\thold}(s) \sum_a \pith(a\given s) \Apithold(s,a) \nonumber \\
&\text{\ \  subject to }
\meankl{\rho_{\thold}}(\thold,\theta) \le \delta.
\label{eq:adv0}
\end{align}

We first replace $\sum_s \rho_{\thold}(s) \lrbrack{\dots}$ in the objective by the expectation $\frac{1}{1-\gamma}\Eb{s \sim \rho_{\thold}}{\dots}$.
Next, we replace the advantage values $\Apithold$ by the $Q$-values $\Qpithold$ in Equation~\eqref{eq:adv0}, which only changes the objective by a constant.
Last, we replace the sum over the actions by an importance sampling estimator. Using $\isd$ to denote the sampling distribution, the contribution of a single $s_n$ to the loss function is
\begin{align}
\sum_a \pith(a \given  s_n) A_{\thold}(s_n,a)
=
\Eb{a \sim \isd}{\frac{\pith(a \given  s_n)}{\isd(a \given  s_n)}A_{\thold}(s_n,a) }.
\nonumber
\end{align}

Our optimization problem in \Cref{eq:adv0} is exactly equivalent to the following one, written in terms of expectations:
\begin{align}
&\maximize_{\theta} \Eb{s \sim \rho_{\thold}, a \sim \isd}{ \frac{\pith(a\given s)}{q(a\given s)} \Qpithold(s,a)}  \label{eq:trprobexp}\\
&\text{\ \  subject to }
\Eb{s \sim \rho_{\thold}}{\kl{\pi_{\thold}(\cdot \given  s)}{\pi_{\theta}(\cdot\given s)}}
\le \delta.  \nonumber
\end{align}

All that remains is to replace the expectations by sample averages and replace the $Q$ value by an empirical estimate.
The following sections describe two different schemes for performing this estimation.

The first sampling scheme, which we call \singlepath, is the one that is typically used for policy gradient estimation \cite{bartlett2001infinite}, and is based on sampling individual trajectories.
The second scheme, which we call \vine, involves constructing a rollout set and then performing multiple actions from each state in the rollout set.
This method has mostly been explored in the context of policy iteration methods \cite{lagoudakis2003reinforcement,ggs-adpfp-13}.

\subsection{Single Path}

In this estimation procedure, we collect a sequence of states by sampling $s_0 \sim \rho_0$ and then simulating the policy $\pi_{\thold}$ for some number of timesteps to generate a trajectory $s_0, a_0, s_1, a_1, \dots, s_{T-1},a_{T-1},s_T$.
Hence, $q(a \given s) = \pi_{\thold}(a \given s)$.
$Q_{\thold}(s,a)$ is computed at each state-action pair $(s_t,a_t)$ by taking the discounted sum of future rewards along the trajectory.

\subsection{Vine}

\begin{figure}
\centering
\includegraphics[height=2.5cm]{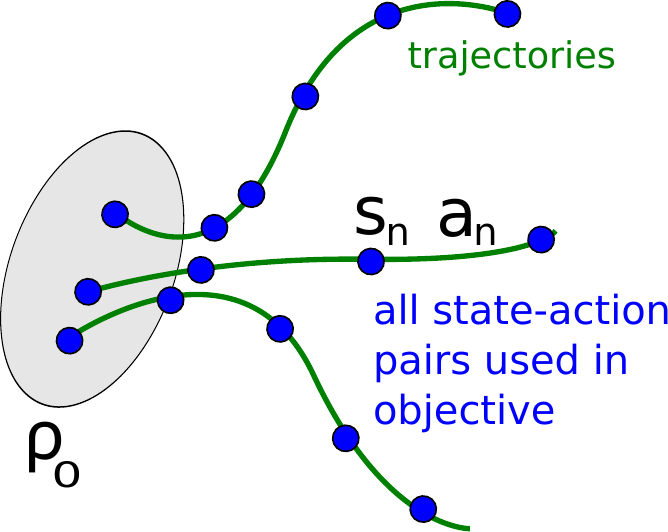}
\includegraphics[height=2.5cm]{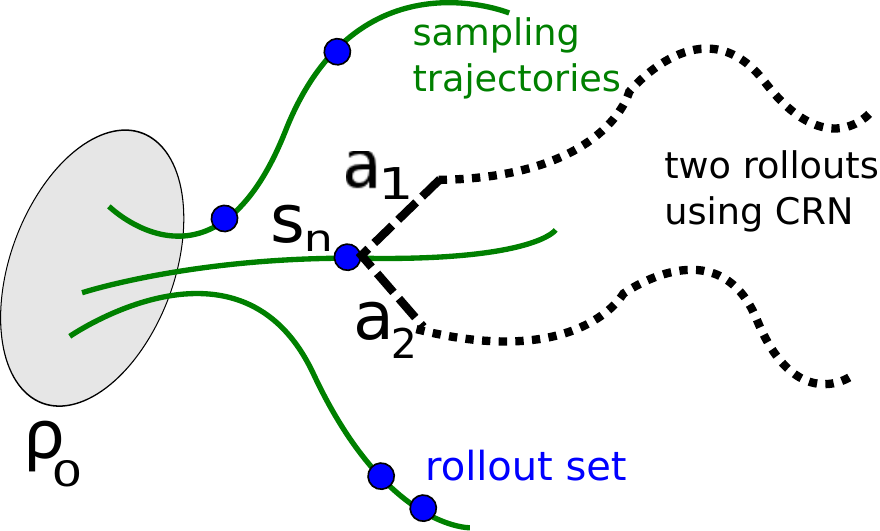}
\caption{Left: illustration of single path procedure. Here, we generate a set of trajectories via simulation of the policy and incorporate all state-action pairs ($s_n,a_n$) into the objective. Right: illustration of vine procedure. We generate a set of ``trunk'' trajectories, and then generate ``branch'' rollouts from a subset of the reached states. For each of these states $s_n$, we perform multiple actions ($a_1$ and $a_2$ here) and perform a rollout after each action, using common random numbers (CRN) to reduce the variance.}
\label{fig:vinesp}
\end{figure}

In this estimation procedure, we first sample $s_0 \sim \rho_0$ and simulate the policy $\pi_{\theta_i}$ to generate a number of trajectories.
We then choose a subset of $N$ states along these trajectories, denoted $s_1, s_2, \dots, s_N$, which we call the ``rollout set''.
For each state $s_n$ in the rollout set, we sample $K$ actions according to $a_{n,k} \sim \isd(\cdot \given  s_n)$. Any choice of $\isd(\cdot \given s_n)$ with a support that includes the support of $\pi_{\theta_i}(\cdot \given s_n)$ will produce a consistent estimator. In practice, we found that $\isd(\cdot \given s_n) = \pi_{\theta_i}(\cdot \given s_n)$ works well on continuous problems, such as robotic locomotion, while the uniform distribution works well on discrete tasks, such as the Atari games, where it can sometimes achieve better exploration.

For each action $a_{n,k}$ sampled at each state $s_n$, we estimate $\hat{Q}_{\theta_i}(s_n, a_{n,k})$ by performing a rollout (i.e., a short trajectory) starting with state $s_n$ and action $a_{n,k}$.
We can greatly reduce the variance of the $Q$-value differences between rollouts by using the same random number sequence for the noise in each of the $K$ rollouts, i.e., \textit{common random numbers}.
See \cite{bert} for additional discussion on Monte Carlo estimation of $Q$-values and \cite{ng2000pegasus} for a discussion of common random numbers in reinforcement learning.

In small, finite action spaces, we can generate a rollout for every possible action from a given state.
The contribution to $L_{\thold}$ from a single state $s_n$ is as follows:
\begin{align}
L_n(\theta) = \sum_{k=1}^K \pith(a_k \given s_n) \hat{Q}(s_n,a_k),
\end{align}
where the action space is $\cA = \lrbrace{a_1,a_2,\dots,a_K}$.
In large or continuous state spaces, we can construct an estimator of the surrogate objective using importance sampling.
The self-normalized estimator (\citet{mcbook}, Chapter 9) of $L_{\thold}$ obtained at a single state $s_n$ is
\begin{align}
L_n(\theta) =
\frac
{\sum_{k=1}^K \frac{\pith(a_{n,k} \given s_n)}{\pi_{\thold}(a_{n,k} \given s_n)} \hat{Q}(s_n,a_{n,k})}
{\sum_{k=1}^K \frac{\pith(a_{n,k} \given s_n)}{\pi_{\thold}(a_{n,k} \given s_n)} },
\end{align}
assuming that we performed $K$ actions $a_{n,1}, a_{n,2}, \dots, a_{n,K}$ from state $s_n$.
This self-normalized estimator removes the need to use a baseline for the $Q$-values (note that the gradient is unchanged by adding a constant to the $Q$-values).
Averaging over $s_n \sim \rho(\pi)$, we obtain an estimator for $L_{\thold}$, as well as its gradient.

The \textit{vine} and \textit{single path} methods are illustrated in \Cref{fig:vinesp}.
We use the term \textit{vine}, since the trajectories used for sampling can be likened to the stems of vines, which branch at various points (the rollout set) into several short offshoots (the rollout trajectories).

The benefit of the \vine{} method over the \singlepath{} method that is our local estimate of the objective has much lower variance given the same number of $Q$-value samples in the surrogate objective.
That is, the \vine{} method gives much better estimates of the advantage values.
The downside of the \vine{} method is that we must perform far more calls to the simulator for each of these advantage estimates.
Furthermore, the \vine{} method requires us to generate multiple trajectories from each state in the rollout set, which limits this algorithm to settings where the system can be reset to an arbitrary state. In contrast, the single path algorithm requires no state resets and can be directly implemented on a physical system \cite{peters2008natural}.

\section{Practical Algorithm} \label{sec:practical}

Here we present two practical policy optimization algorithm based on the ideas above, which use either the \singlepath{} or \vine{} sampling scheme from the preceding section.
The algorithms repeatedly perform the following steps:
\begin{enumerate}
\item Use the \singlepath{} or \vine{} procedures to collect a set of state-action pairs along with Monte Carlo estimates of their $Q$-values.
\item By averaging over samples, construct the estimated objective and constraint in \Cref{eq:trprobexp}.
\item Approximately solve this constrained optimization problem to update the policy's parameter vector $\theta$.
We use the conjugate gradient algorithm followed by a line search, which is altogether only slightly more expensive than computing the gradient itself. See \Cref{sec:cg} for details.
\end{enumerate}

With regard to (3), we construct the Fisher information matrix (FIM) by analytically computing the Hessian of the KL divergence, rather than using the covariance matrix of the gradients.
That is, we estimate $A_{ij}$ as \mbox{$\frac{1}{N}\sum_{n=1}^N \frac{\partial^2}{\partial\theta_i \partial \theta_j} \kl{\pi_{\thold}(\cdot \given s_n)}{\pi_{\theta}(\cdot \given s_n)}$}, rather than \mbox{$\frac{1}{N} \sum_{n=1}^N \frac{\partial}{\partial\theta_i}\log \pi_{\theta}(a_n \given s_n) \frac{\partial}{\partial\theta_j}\log \pi_{\theta}(a_n \given s_n) $}.
The analytic estimator integrates over the action at each state $s_n$, and does not depend on the action $a_n$ that was sampled.
As described in \Cref{sec:cg}, this analytic estimator has computational benefits in the large-scale setting, since it removes the need to store a dense Hessian or all policy gradients from a batch of trajectories. The rate of improvement in the policy is similar to the empirical FIM, as shown in the experiments.

Let us briefly summarize the relationship between the theory from \Cref{sec:improvetheory} and the practical algorithm we have described:
\begin{itemize}
\item The theory justifies optimizing a surrogate objective with a penalty on KL divergence. However, the large penalty coefficient $C$ leads to prohibitively small steps, so we would like to decrease this coefficient.
Empirically, it is hard to robustly choose the penalty coefficient, so we use a hard constraint instead of a penalty, with parameter $\delta$ (the bound on KL divergence).
\item The constraint on $\maxkl(\thold,\theta)$ is hard for numerical optimization and estimation, so instead we constrain $\meankl{}(\thold,\theta)$.
\item Our theory ignores estimation error for the advantage function.
\citet{kakade2002approximately} consider this error in their derivation, and the same arguments would hold in the setting of this paper, but we omit them for simplicity.
\end{itemize}

\section{Connections with Prior Work}
\label{sec:related}

As mentioned in Section~\ref{sec:parameterized}, our derivation results in a policy update that is related to several prior methods, providing a unifying perspective on a number of policy update schemes. The natural policy gradient \cite{kakade2002natural} can be obtained as a special case of the update in Equation~(\ref{eq:trprob}) by using a linear approximation to $L$ and a quadratic approximation to the $\meankl{}$ constraint, resulting in the following problem:
\begin{align}
&\maximize_{\theta} \lrbrack{ \gradth L_{\thold}(\theta)\evalat{\theta=\thold} \cdot (\theta - \thold) } \label{eq:natpolgrad} \\ &\text{\ \  subject to }  \half (\thold - \theta)^T A(\thold) (\thold - \theta) \le \delta,
\nonumber\\
& \text{\ \ where } A(\thold)_{ij} = \nonumber \\
& \quad \frac{\partial}{\partial\theta_i}\frac{\partial}{\partial\theta_j}   \Eb{s \sim \rhopi} { \kl{\pi(\cdot \given  s, \thold)}{\pi(\cdot \given  s, \theta)} }\evalat{\theta=\thold}.\nonumber
\end{align}
The update is \mbox{$\thnew = \thold + \frac{1}{\lambda} A(\thold)^{-1} \gradth L(\theta)\evalat{\theta=\thold}$}, where the stepsize $\frac{1}{\lambda}$ is typically treated as an algorithm parameter. This differs from our approach, which enforces the constraint at each update. Though this difference might seem subtle, our experiments demonstrate that it significantly improves the algorithm's performance on larger problems.

We can also obtain the standard policy gradient update by using an $\ell_2$ constraint or penalty:
\begin{align}
&\maximize_{\theta}  \lrbrack{\gradth L_{\thold}(\theta)\evalat{\theta=\thold} \cdot (\theta - \thold)} \label{eq:vanillapg} \\ &\text{\ \  subject to }  \half \norm{\theta - \thold}^2 \le \delta.  \nonumber
\end{align}
The policy iteration update can also be obtained by solving the unconstrained problem $\maximize_{\pi} L_{\piold}(\pi)$, using $L$ as defined in \Cref{eq:adv}.

Several other methods employ an update similar to Equation~(\ref{eq:trprob}). \textit{Relative entropy policy search} (REPS) \cite{pma-reps-10} constrains the state-action marginals $p(s, a)$, while TRPO constrains the conditionals $p(a \given s)$. Unlike REPS, our approach does not require a costly nonlinear optimization in the inner loop.
Levine and Abbeel \yrcite{levine2014learning} also use a KL divergence constraint, but its purpose is to encourage the policy not to stray from regions where the estimated dynamics model is valid, while we do not attempt to estimate the system dynamics explicitly.
\citet{pirotta2013safe} also build on and generalize Kakade and Langford's results, and they derive different algorithms from the ones here.

\textremovedforicml{
To review, a large portion of the pantheon of policy optimization algorithms can be understood as optimizing the surrogate objective function $L_{\thold}$, possibly subject to a constraint:
\begin{itemize}
\item The natural policy gradient is the step direction obtained by using a KL divergence constraint and shrinking the bound $\delta$ to zero, or equivalently making a quadratic approximation to the constraint function.
\item Conservative policy iteration uses a particular policy representation scheme to reduce the surrogate objective while obeying a constraint ($\maxtv < \delta$) on the policy.
\item The ``vanilla'' policy gradient is obtained by using a trust region (or penalty term) on the Euclidean distance $\norm{\theta-\thold}$.
\item The classic approximate policy iteration algorithm \cite{bertsekas2011approximate} is obtained by removing the constraint and fully minimizing $L_{\thold}$ in each batch.
\end{itemize}
}

\section{Experiments}\label{sec:expt}

We designed our experiments to investigate the following questions:
\begin{enumerate}
\vspace{-0.15in}
\item What are the performance characteristics of the \singlepath{} and \vine{} sampling procedures?
\item TRPO is related to prior methods (e.g. natural policy gradient) but makes several changes, most notably by using a fixed KL divergence rather than a fixed penalty coefficient. How does this affect the performance of the algorithm?
\item Can TRPO be used to solve challenging large-scale problems? How does TRPO compare with other methods when applied to large-scale problems, with regard to final performance, computation time, and sample complexity?
\vspace{-0.15in}
\end{enumerate}
To answer (1) and (2), we compare the performance of the \singlepath{} and \vine{} variants of TRPO, several ablated variants, and a number of prior policy optimization algorithms.
With regard to (3), we show that both the \singlepath{} and \vine{} algorithm can obtain high-quality locomotion controllers from scratch, which is considered to be a hard problem.
We also show that these algorithms produce competitive results when learning policies for playing Atari games from images using convolutional neural networks with tens of thousands of parameters.

\begin{figure}
\centering
\includegraphics[width=.4\textwidth]{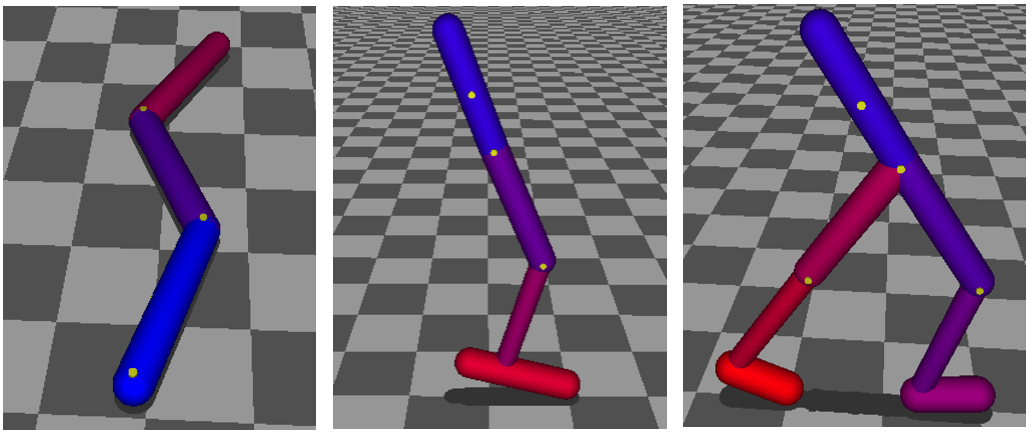}
\caption{2D robot models used for locomotion experiments. From left to right: swimmer, hopper, walker. The hopper and walker present a particular challenge, due to underactuation and contact discontinuities.}
\label{fig:mjcmodels}
\vspace{-0.15in}
\end{figure}

\subsection{Simulated Robotic Locomotion}

\begin{figure}
\centering
\includegraphics[height=3.5cm]{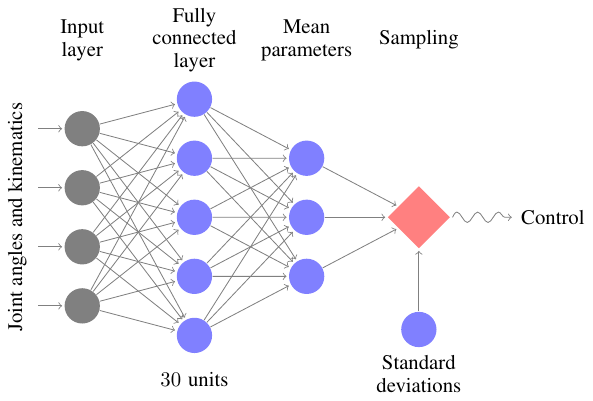}
\includegraphics[height=3.5cm]{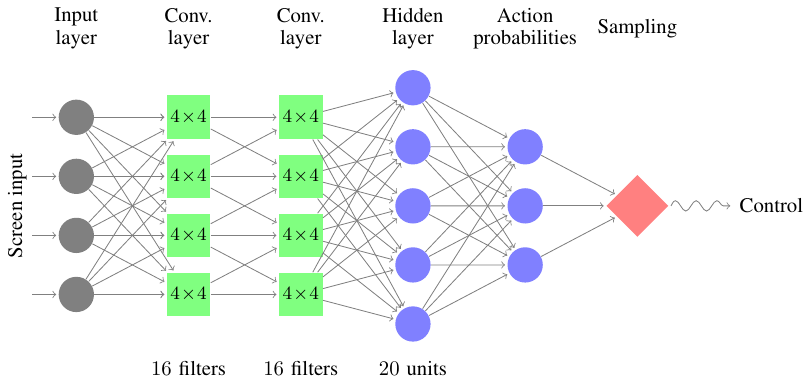}
\vspace{-0.1in}
\caption{Neural networks used for the locomotion task (top) and for playing Atari games (bottom).}
\label{fig:nn}
\vspace{-0.15in}
\end{figure}

We conducted the robotic locomotion experiments using the MuJoCo simulator \cite{todorov2012mujoco}.
The three simulated robots are shown in \Cref{fig:mjcmodels}.
The states of the robots are their generalized positions and velocities, and the controls are joint torques. Underactuation, high dimensionality, and non-smooth dynamics due to contacts make these tasks very challenging.
The following models are included in our evaluation:
\begin{enumerate}
\vspace{-0.1in}
\item \textit{Swimmer}. $10$-dimensional state space, linear reward for forward progress and a quadratic penalty on joint effort to produce the reward $r(x,u) = v_x - 10^{-5}\norm{u}^2$. The swimmer can propel itself forward by making an undulating motion.
\item \textit{Hopper}. $12$-dimensional state space, same reward as the swimmer, with a bonus of $+1$ for being in a non-terminal state.
We ended the episodes when the hopper fell over, which was defined by thresholds on the torso height and angle.
\item \textit{Walker}. $18$-dimensional state space. For the walker, we added a penalty for strong impacts of the feet against the ground to encourage a smooth walk rather than a hopping gait.
\vspace{-0.1in}
\end{enumerate}

We used $\delta=0.01$ for all experiments.
See \Cref{table:contctrlparams} in the Appendix for more details on the experimental setup and parameters used.
We used neural networks to represent the policy, with the architecture shown in \Cref{fig:nn}, and further details provided in \Cref{sec:nn}. To establish a standard baseline, we also included the classic cart-pole balancing problem, based on the formulation from \citet{barto1983neuronlike}, using a linear policy with six parameters that is easy to optimize with derivative-free black-box optimization methods.

The following algorithms were considered in the comparison:
\textit{single path TRPO};
\textit{vine TRPO};
\textit{cross-entropy method} (CEM), a gradient-free method \cite{szita2006learning};
\textit{covariance matrix adaption} (CMA), another gradient-free method \cite{hansen1996adapting};
\textit{natural gradient}, the classic natural policy gradient algorithm \cite{kakade2002natural}, which differs from \textit{single path} by the use of a fixed penalty coefficient (Lagrange multiplier) instead of the KL divergence constraint;
\textit{empirical FIM}, identical to \textit{single path}, except that the FIM is estimated using the covariance matrix of the gradients rather than the analytic estimate;
\textit{max KL}, which was only tractable on the cart-pole problem, and uses the maximum KL divergence in Equation~(\ref{eq:maxklconst}), rather than the average divergence, allowing us to evaluate the quality of this approximation.
The parameters used in the experiments are provided in \Cref{sec:exptparams}.
For the \textit{natural gradient} method, we swept through the possible values of the stepsize in factors of three, and took the best value according to the final performance.

\begin{figure}
\centering
{\includegraphics[width=0.23\textwidth]{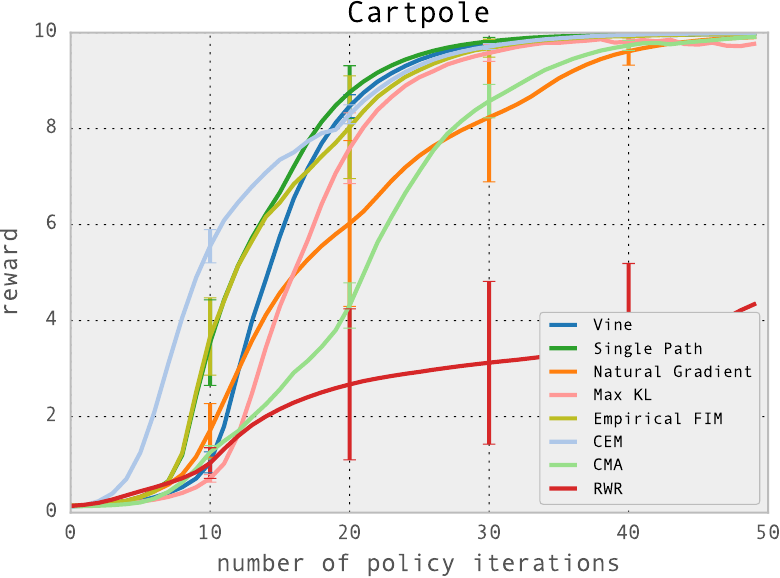}}
{\includegraphics[width=0.23\textwidth]{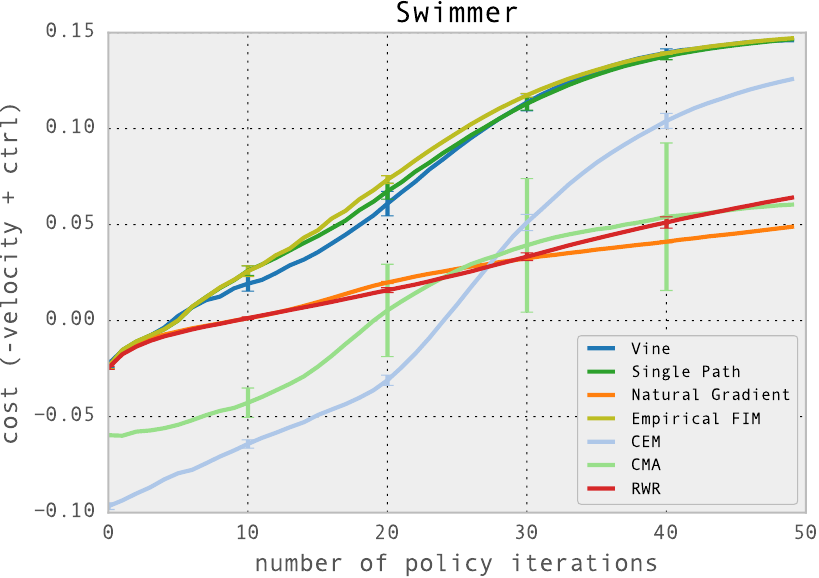}}
\\
{\includegraphics[width=0.23\textwidth]{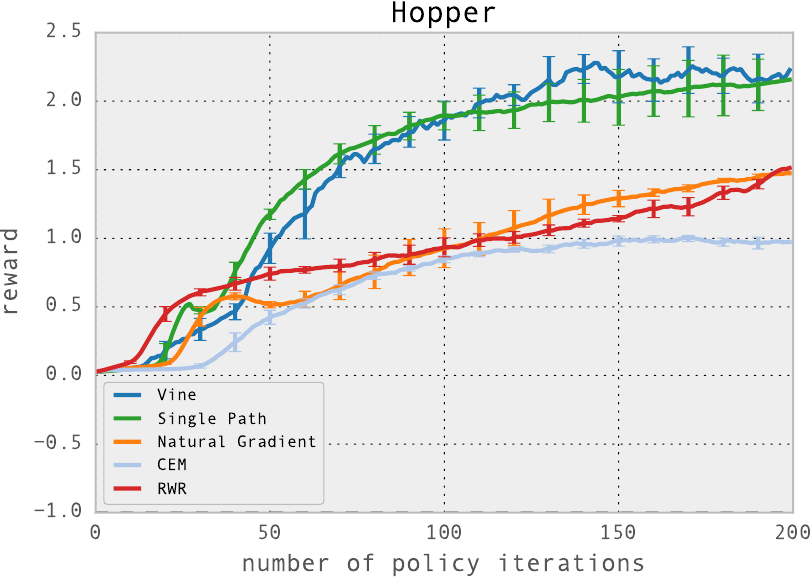}}
{\includegraphics[width=0.23\textwidth]{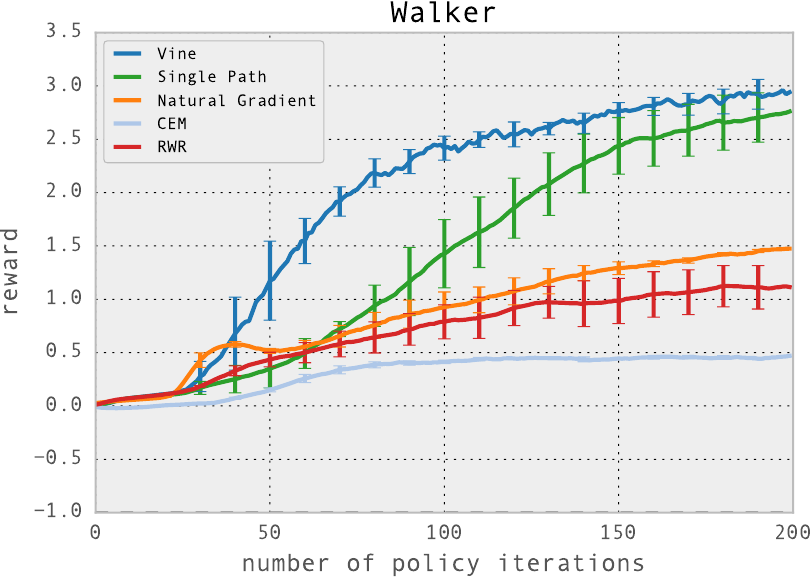}}
\caption{Learning curves for locomotion tasks, averaged across five runs of each algorithm with random initializations. Note that for the hopper and walker, a score of $-1$ is achievable without any forward velocity, indicating a policy that simply learned balanced standing, but not walking.
}
\label{fig:locomotion}
\vspace{-0.15in}
\end{figure}

Learning curves showing the total reward averaged across five runs of each algorithm are shown in \Cref{fig:locomotion}.
\textit{Single path} and \textit{vine} TRPO solved all of the problems, yielding the best solutions.
\textit{Natural gradient} performed well on the two easier problems, but was unable to generate hopping and walking gaits that made forward progress.
These results provide empirical evidence that constraining the KL divergence is a more robust way to choose step sizes and make fast, consistent progress, compared to using a fixed penalty.
CEM and CMA are derivative-free algorithms, hence their sample complexity scales unfavorably with the number of parameters, and they performed poorly on the larger problems.
The \textit{max KL} method learned somewhat more slowly than our final method, due to the more restrictive form of the constraint, but overall the result suggests that the average KL divergence constraint has a similar effect as the theorecally justified maximum KL divergence. Videos of the policies learned by TRPO may be viewed on the project website: \url{http://sites.google.com/site/trpopaper/}.

Note that TRPO learned all of the gaits with general-purpose policies and simple reward functions, using minimal prior knowledge. This is in contrast with most prior methods for learning locomotion, which typically rely on hand-architected policy classes that explicitly encode notions of balance and stepping \cite{tzs-spgrl-04,gpw-fbwrc-06,wampler2009optimal}.

\begin{table*}
\tiny
\centering
\begin{tabular}{lrrrrrrr}
\toprule
& {\it B. Rider} & {\it Breakout} & {\it Enduro} & {\it Pong} & {\it Q*bert} & {\it Seaquest} & {\it S. Invaders}\\
\midrule
Random      & 354  & 1.2   & 0   & $-20.4$  & 157  & 110  & 179 \\
Human \cite{mnih2013playing}       & 7456 & 31.0  & 368 & $-3.0$   & 18900 & 28010 & 3690    \\
\midrule
Deep Q Learning \cite{mnih2013playing}   & 4092 & 168.0 & 470 & 20.0     & 1952 & 1705 & 581     \\
\midrule
UCC-I \cite{guo2014deep}  & 5702 & 380 & 741 & 21 & 20025 & 2995 & 692 \\
\midrule
TRPO - single path & 1425.2 &  10.8 &  534.6 &  20.9 &  1973.5 &  1908.6 &  568.4 \\
TRPO - vine    & 859.5 &  34.2 &  430.8 &  20.9 &  7732.5 &  788.4 &  450.2 \\
\bottomrule
\end{tabular}
\caption{Performance comparison for vision-based RL algorithms on the \Atari{} domain.
Our algorithms (bottom rows) were run once on each task, with the same architecture and parameters.
Performance varies substantially from run to run (with different random initializations of the policy), but we could not obtain error statistics due to time constraints.
\label{table:atariresults}}
\end{table*}

\subsection{Playing Games from Images}

To evaluate TRPO on a partially observed task with complex observations, we trained policies for playing  \Atari{} games, using raw images as input. The games require learning a variety of behaviors, such as dodging bullets and hitting balls with paddles.
Aside from the high dimensionality, challenging elements of these games include delayed rewards (no immediate penalty is incurred when a life is lost in Breakout or Space Invaders); complex sequences of behavior (Q*bert requires a character to hop on $21$ different platforms); and non-stationary image statistics (Enduro involves a changing and flickering background).

We tested our algorithms on the same seven games reported on in \cite{mnih2013playing} and \cite{guo2014deep}, which are made available through the Arcade Learning Environment \cite{bellemare13arcade}
The images were preprocessed following the protocol in Mnih et al \yrcite{mnih2013playing}, and
the policy was represented by the convolutional neural network shown in \Cref{fig:nn}, with two convolutional layers with $16$ channels and stride $2$, followed by one fully-connected layer with $20$ units, yielding 33,500 parameters.

The results of the \vine{} and \singlepath{} algorithms are summarized in \Cref{table:atariresults}, which also includes an expert human performance and two recent methods: deep $Q$-learning \cite{mnih2013playing}, and a combination of Monte-Carlo Tree Search with supervised training \cite{guo2014deep}, called UCC-I. The 500 iterations of our algorithm took about 30 hours (with slight variation between games) on a 16-core computer. While our method only outperformed the prior methods on some of the games, it consistently achieved reasonable scores. Unlike the prior methods, our approach was not designed specifically for this task. The ability to apply the same policy search method to methods as diverse as robotic locomotion and image-based game playing demonstrates the generality of TRPO.

\section{Discussion}

We proposed and analyzed trust region methods for optimizing stochastic control policies.
We proved monotonic improvement for an algorithm that repeatedly optimizes a local approximation to the expected return of the policy with a KL divergence penalty, and we showed that an approximation to this method that incorporates a KL divergence constraint achieves good empirical results on a range of challenging policy learning tasks, outperforming prior methods. Our analysis also provides a perspective that unifies policy gradient and policy iteration methods, and shows them to be special limiting cases of an algorithm that optimizes a certain objective subject to a trust region constraint.

In the domain of robotic locomotion, we successfully learned controllers for swimming, walking and hopping in a physics simulator, using general purpose neural networks and minimally informative rewards.
To our knowledge, no prior work has learned controllers from scratch for all of these tasks, using a generic policy search method and non-engineered, general-purpose policy representations.
In the game-playing domain, we learned convolutional neural network policies that used raw images as inputs. This requires optimizing extremely high-dimensional policies, and only two prior methods report successful results on this task.

Since the method we proposed is scalable and has strong theoretical foundations, we hope that it will serve as a jumping-off point for future work on training large, rich function approximators for a range of challenging problems.
At the intersection of the two experimental domains we explored, there is the possibility of learning robotic control policies that use vision and raw sensory data as input, providing a unified scheme for training robotic controllers that perform both perception and control.
The use of more sophisticated policies, including recurrent policies with hidden state, could further make it possible to roll state estimation and control into the same policy in the partially-observed setting.
By combining our method with model learning, it would also be possible to substantially reduce its sample complexity, making it applicable to real-world settings where samples are expensive.

\section*{Acknowledgements}
We thank Emo Todorov and Yuval Tassa for providing the MuJoCo simulator; Bruno Scherrer, Tom Erez, Greg Wayne, and the anonymous ICML reviewers for insightful comments, and Vitchyr Pong and Shane Gu for pointing our errors in a previous version of the manuscript. This research was funded in part by the Office of Naval Research through a Young Investigator Award and under grant number N00014-11-1-0688,
DARPA through a Young Faculty Award, by the Army Research
Office through the MAST program.

{
\bibliographystyle{icml2015}
\small
\setlength{\bibsep}{6pt}
\bibliography{policyopt}
}

\appendix{
\onecolumn

\section{Proof of Policy Improvement Bound} \label{sec:couplingproof}

This proof (of \Cref{thm:impthm}) uses techniques from the proof of Theorem 4.1 in \cite{kakade2002approximately}, adapting them to the more general setting considered in this paper.
An informal overview is as follows.
Our proof relies on the notion of coupling, where we jointly define the policies $\pi$ and $\pi'$ so that they choose the same action with high probability $=(1-\alpha)$.
Surrogate loss $L_{\pi}(\tilpi)$ accounts for the the advantage of $\tilpi$ the first time that it disagrees with $\pi$, but not subsequent disagreements. Hence, the error in $L_{\pi}$ is due to two or more disagreements between $\pi$ and $\tilpi$, hence, we get an $O(\alpha^2)$ correction term, where $\alpha$ is the probability of disagreement.

We start out with a lemma from \citet{kakade2002approximately} that shows that the difference in policy performance $\eta(\tilpi) - \eta(\pi)$ can be decomposed as a sum of per-timestep advantages.
\begin{lemma} \label{improvelemma} Given two policies $\pi,\tilpi$,
\begin{align}
\eta(\tilpi) = \eta(\pi) +  &\Eb{\tau \sim \tilpi}{\sum_{t=0}^{\infty} \gamma^t A_{\pi}(s_t, a_t)} \nonumber \\
\end{align}
This expectation is taken over trajectories $\tau \defeq (s_0, a_0, s_1, a_0, \dots)$, and the notation $\Eb{\tau \sim \tilpi}{\dots}$ indicates that actions are sampled from $\tilpi$ to generate $\tau$.
\end{lemma}
\begin{proof}
First note that $A_{\pi}(s,a) = \Eb{s' \sim P(s' \given s,a)}{r(s) + \gamma \Vpi(s') - \Vpi(s)}$.
Therefore,
\begin{align}
&\Eb{\tau \given \tilpi}{\sum_{t=0}^{\infty} \gamma^t A_{\pi}(s_t, a_t)}\\
&= \Eb{\tau \given \tilpi}{\sum_{t=0}^{\infty} \gamma^t (r(s_t) + \gamma \Vpi(s_{t+1}) - \Vpi(s_t))}\\
&= \Eb{\tau \given \tilpi}{-\Vpi(s_0)+ \sum_{t=0}^{\infty} \gamma^t r(s_t)}  \\
&= -\Eb{s_0}{\Vpi(s_0)}+\Eb{\tau \given \tilpi}{\sum_{t=0}^{\infty} \gamma^t r(s_t)}  \label{eq:telescope}\\
&=  -\eta(\pi) + \eta(\tilpi)
\end{align}
Rearranging, the result follows.
\end{proof}

\newcommand{\meanadv}{\bar{A}}
Define $\meanadv(s)$ to be the expected advantage of $\tilpi$ over $\pi$ at state $s$:
\begin{align}
\meanadv(s) = \Eb{a \sim \tilpi(\cdot \given s)}{\Api(s, a)}.
\end{align}

Now \Cref{improvelemma} can be written as follows:
\begin{align}
\eta(\tilpi) = \eta(\pi) + \Eb{\tau \sim \tilpi}{\sum_{t=0}^{\infty} \gamma^t \meanadv(s_t)} \label{etadef}
\end{align}
Note that $L_{\pi}$ can be written as
\begin{align}
L_{\pi}(\tilpi) = \eta(\pi) + \Eb{\tau \sim \pi}{\sum_{t=0}^{\infty} \gamma^t \meanadv(s_t)} \label{tilpidef}
\end{align}
The difference in these equations is whether the states are sampled using $\pi$ or $\tilpi$.
To bound the difference between $\eta(\tilpi)$ and $L_{\pi}(\tilpi)$, we will bound the difference arising from each timestep.
To do this, we first need to introduce a measure of how much $\pi$ and $\tilpi$ agree.
Specifically, we'll \textit{couple} the policies, so that they define a joint distribution over pairs of actions.
\begin{defn}
$(\pi,\tilpi)$ is an \textit{$\alpha$-coupled policy pair} if it defines a joint distribution $(a,\tilde{a}) \given s$, such that $P(a \neq \tilde{a} \given s) \le \alpha$ for all $s$.
$\pi$ and $\tilpi$ will denote the marginal distributions of $a$ and $\tilde{a}$, respectively.
\end{defn}
Computationally, $\alpha$-coupling means that if we randomly choose a seed for our random number generator, and then  we sample from each of $\pi$ and $\tilpi$ after setting that seed, the results will agree for at least fraction $1 - \alpha$ of seeds.

\begin{lemma}
Given that $\pi, \tilpi$ are $\alpha$-coupled policies, for all $s$,
\begin{align}
\abs*{\meanadv(s)} \le 2 \alpha \maxadv
\end{align}
\end{lemma}
\begin{proof}
\begin{align}
\meanadv(s) &= \Eb{\tilde{a} \sim \tilpi}{\Api(s,\tilde{a})} =\Eb{(a,\tilde{a}) \sim (\pi,\tilpi)}{\Api(s,\tilde{a}) - \Api(s, a)} \quad\text{since \quad$\Eb{a \sim \pi}{\Api(s,a)}=0$}\\
&= P(a \neq \tilde{a} \given s) \Eb{(a,\tilde{a}) \sim (\pi,\tilpi) \given a \neq \tilde{a}}{\Api(s,\tilde{a}) - \Api(s, a)}\\
\abs{\meanadv(s)} &\le \alpha \cdot 2 \maxadv
\end{align}
\end{proof}

\begin{lemma}
\label{advbound}
Let $(\pi,\tilpi)$ be an $\alpha$-coupled policy pair. Then
\begin{align}
\abs*{\Eb{s_t \sim \tilpi}{\meanadv(s_t)} - \Eb{s_t \sim \pi}{\meanadv(s_t)}} &\le 2 \alpha \max_s \meanadv(s)\le 4 \alpha (1 - (1 - \alpha)^t)\max_s \abs{\Api(s, a)}
\end{align}
\end{lemma}
\begin{proof}

\newcommand{\st}{s_t}
\newcommand{\nt}{n_t}
\DeclarePairedDelimiter{\lrparena}{(}{)}

Given the coupled policy pair $(\pi, \tilpi)$, we can also obtain a coupling over the trajectory distributions produced by $\pi$ and $\tilpi$, respectively.
Namely, we have pairs of trajectories $\tau, \tilde{\tau}$, where $\tau$ is obtained by taking actions from $\pi$, and $\tilde{\tau}$ is obtained by taking actions from $\tilpi$, where the same random seed is used to generate both trajectories.
We will consider the advantage of $\tilpi$ over $\pi$ at timestep $t$, and decompose this expectation based on whether $\pi$ agrees with $\tilpi$ at all timesteps $i<t$.

Let $n_t$ denote the number of times that $a_i \neq \tilde{a}_i$ for $i < t$, i.e., the number of times that $\pi$ and $\tilpi$ disagree before timestep $t$.
\begin{align}
\Eb{s_t \sim \tilpi}{\meanadv(s_t)}
= P(n_t=0)\Eb{s_t \sim \tilpi \given n_t = 0}{\meanadv(s_t)}
+ P(n_t>0)\Eb{s_t \sim \tilpi \given n_t > 0}{\meanadv(s_t)}
\label{decomp-pitil}
\end{align}
The expectation decomposes similarly for actions are sampled using $\pi$:
\begin{align}
\Eb{\st \sim \pi}{\meanadv(\st)}
=
P(n_t = 0)
\Eb{\st \sim \pi | \nt = 0}{\meanadv(\st)}
+
P(n_t > 0)
\Eb{\st \sim \pi | \nt > 0}{\meanadv(\st)} \label{decomp-pi}
\end{align}
Note that the $n_t=0$ terms are equal:
\begin{align}
\Eb{s_t \sim {\color{red} \tilpi} \given n_t = 0}{\meanadv(s_t)}=\Eb{s_t \sim {\color{red} \pi} \given n_t = 0}{\meanadv(s_t)},
\end{align}
because $n_t=0$ indicates that $\pi$ and $\tilpi$ agreed on all timesteps less than $t$.
Subtracting \Cref{decomp-pitil,decomp-pi}, we get
\begin{align}
\Eb{\st \sim \tilpi}{\meanadv(s_t)} - \Eb{\st \sim \pi}{\meanadv(s_t)}
&=
P(n_t > 0)
\lrparena*{
\Eb{\st \sim \tilpi | \nt > 0}{\meanadv(s_t)}
- \Eb{\st \sim \pi | \nt > 0}{\meanadv(s_t)}
}
\label{diff-ineq}
\end{align}
By definition of $\alpha$, $P(\pi,\tilpi \text{ agree at timestep $i$}) \ge 1-\alpha$, so $P(n_t = 0) \ge (1-\alpha)^t$, and
\begin{align}
P(n_t > 0) \le 1 - (1 - \alpha)^t  \label{nineq}
\end{align}
Next, note that
\begin{align}
\abs*{\Eb{\st \sim \tilpi | \nt > 0}{\meanadv(s_t)} - \Eb{\st \sim \pi | \nt > 0}{\meanadv(s_t)}}
&\le \abs*{\Eb{\st \sim \tilpi | \nt > 0}{\meanadv(s_t)}} +\abs*{\Eb{\st \sim \pi | \nt > 0}{\meanadv(s_t)}}\\
&\le 4\alpha \maxadv \label{epsineq}
\end{align}
Where the second inequality follows from \Cref{advbound}.

Plugging \Cref{nineq} and \Cref{epsineq} into \Cref{diff-ineq}, we get
\begin{align}
\abs*{\Eb{\st \sim \tilpi}{\meanadv(s_t)}
-
\Eb{\st \sim \pi}{\meanadv(s_t)}} \le 4 \alpha (1 - (1 - \alpha)^t) \maxadv
\end{align}
\end{proof}

The preceding Lemma bounds the difference in expected advantage at each timestep $t$. We can sum over time to bound the difference between $\eta(\tilpi)$ and $L_{\pi}(\tilpi)$.
Subtracting \Cref{etadef} and \Cref{tilpidef}, and defining $\epsilon=\maxadv$,
\begin{align}
\abs*{\eta(\tilpi) - L_{\pi}(\tilpi)}
&=\sum_{t=0}^{\infty} \gamma^t \abs*{
\Eb{\tau \sim \tilpi}{\meanadv(s_t)}-\Eb{\tau \sim \pi}{\meanadv(s_t)}
}\\
&\le \sum_{t=0}^{\infty} \gamma^t \cdot 4\epsilon \alpha (1 - (1 - \alpha)^t)\\
&= 4\epsilon \alpha \lrparen{\frac{1}{1 - \gamma} - \frac{1}{1 - \gamma(1- \alpha)}}\\
&= \frac{4\alpha^2 \gamma \epsilon}{(1 - \gamma)(1 - \gamma(1- \alpha ))}\\
&\le \frac{4\alpha^2 \gamma \epsilon}{(1 - \gamma)^2}
\label{fracsineq}
\end{align}

Last, to replace $\alpha$ by the total variation divergence, we need to use the correspondence between TV divergence and coupled random variables:
\begin{quote}
Suppose $p_X$ and $p_Y$ are distributions with $\tv{p_X}{p_Y} = \alpha$.
Then there exists a joint distribution $(X,Y)$ whose marginals are $p_X, p_Y$, for which $X = Y$ with probability $1-\alpha$.
\end{quote}
See \cite{levin2009markov}, Proposition 4.7.

It follows that if we have two policies $\pi$ and $\tilpi$ such that $\max_s \tv{\pi(\cdot \given  s)}{\tilpi(\cdot \given  s) } \le \alpha$, then we can define an $\alpha$-coupled policy pair $(\pi,\tilpi)$ with appropriate marginals.
Taking $\alpha=\max_s \tv{\pi(\cdot \given  s)}{\tilpi(\cdot \given  s) } \le \alpha$ in \Cref{fracsineq},
\Cref{thm:impthm} follows.

\section{Perturbation Theory Proof of Policy Improvement Bound}
\label{sec:pert}
We also provide an alternative proof of \Cref{thm:impthm} using perturbation theory.
\begin{proof}
Let $G = (1 + \gamma \Ppi + (\gamma \Ppi)^2 + \dots) = (1-\gamma \Ppi)^{-1}$, and similarly
Let $\tilG = (1 + \gamma \Ppitil + (\gamma \Ppitil)^2 + \dots) = (1-\gamma \Ppitil)^{-1}$.
We will use the convention that $\rho$ (a density on state space) is a vector and $r$ (a reward function on state space) is a dual vector (i.e., linear functional on vectors), thus $r \rho$ is a scalar meaning the expected reward under density $\rho$.
Note that $\eta(\pi) = rG\rho_0$, and $\eta(\tilpi) = c \tilG \rho_0$.
Let $\dP = \Ppitil - \Ppi$.
We want to bound $\eta(\tilpi) - \eta(\pi) = r(\tilG - G)\rho_0$.
We start with some standard perturbation theory manipulations.
\begin{align}
G^{-1} - \tilG^{-1}
&= (1-\gamma \Ppi) - (1-\gamma \Ppitil) \nonumber \\
&= \gamma \dP.
\end{align}
Left multiply by $G$ and right multiply by $\tilG$.
\begin{align}
\tilG - G &= \gamma G \dP\tilG \nonumber \\
\tilG &= G + \gamma G \dP\tilG
\end{align}
Substituting the right-hand side into $\tilG$ gives
\begin{align}
\tilG   &= G + \gamma G \dP G + \gamma^2 G \dP G \dP \tilG
\end{align}
So we have
\begin{align}
\eta(\tilpi) - \eta(\pi) = r(\tilG - G)\rho = \gamma r G \dP G \rho_0 + \gamma^2 r G \dP G \dP \tilG \rho_0
\end{align}

Let us first consider the leading term $\gamma r G \dP G \rho_0$.
Note that $rG=v$, i.e., the infinite-horizon state-value function.
Also note that $G\rho_0 = \rho_{\pi}$.
Thus we can write $\gamma c G \dP G \rho_0 = \gamma v \dP \rho_{\pi}$.
We will show that this expression equals the expected advantage $L_{\pi}(\tilpi) - L_{\pi}(\pi)$.
\begin{align}
L_{\pi}(\tilpi) - L_{\pi}(\pi)
&= \sum_{s} \rhopi(s) \sum_a (\tilpi(a \given  s) - \pi(a \given  s))\Api(s,a)  \nonumber \\
&= \sum_{s} \rhopi(s) \sum_a \lrparen{\pi_{\theta}(a\given s) - \pi_{\tilth}(a \given  s)} \lrbrack{ r(s) + \sum_{s'} p(s'\given s,a) \gamma v(s') - v(s)} \nonumber \\
&= \sum_{s} \rhopi(s)  \sum_{s'} \sum_a \lrparen{\pi(a\given s) - \tilpi(a \given  s)}  p(s'\given s,a) \gamma v(s')  \nonumber \\
&= \sum_{s} \rhopi(s)  \sum_{s'} (p_{\pi}(s'\given s)-p_{\tilpi}(s'\given s)) \gamma v(s')  \nonumber \\
&= \gamma v \dP \rhopi
\end{align}

Next let us bound the $O(\dP^2)$ term $\gamma^2 r G \Delta G \Delta \tilG \rho$.
First we consider the product $ \gamma r G \Delta = \gamma v \Delta$.
Consider the component $s$ of this dual vector.
\begin{align}
\abs{(\gamma v \Delta)_s}
&= \abs*{\sum_a (\tilpi(s,a) - \pi(s,a)) \Qpi(s,a)} \nonumber \\
&= \abs*{\sum_a (\tilpi(s,a) - \pi(s,a)) \Api(s,a)} \nonumber \\
&\le \sum_a \abs*{\tilpi(s,a) - \pi(s,a)} \cdot \max_a \abs{\Api(s,a)} \nonumber \\
&\le 2 \alpha \epsilon
\end{align}
where the last line used the definition of the total-variation divergence, and the definition of $\epsilon=\maxadv$. We bound the other portion $G \Delta \tilG \rho$ using the $\ell_1$ operator norm
\begin{align}
\normone{A} = \sup_{\rho} \lrbrace{ \frac{\normone{A\rho}}{\normone{\rho}} }
\end{align}
where we have that $\normone{G} = \normone{\tilG} = 1/(1-\gamma)$ and $\normone{\Delta}=2\alpha$.
That gives
\begin{align}
\norm{G \Delta \tilG \rho}_1
&\le \normone{G} \normone{\Delta} \normone{\tilG} \normone{\rho} \nonumber \\
&= \frac{1}{1-\gamma} \cdot 2\alpha \cdot \frac{1}{1-\gamma} \cdot 1
\end{align}
So we have that
\begin{align}
\gamma^2 \abs*{r G \Delta G \Delta \tilG \rho}
&\le \gamma \norm{\gamma r G \Delta}_{\infty} \normone{G \Delta \tilG \rho} \nonumber \\
&\le \gamma \norm{v \Delta}_{\infty} \normone{G \Delta \tilG \rho} \nonumber \\
&\le \gamma \cdot 2\alpha \epsilon \cdot \frac{2\alpha}{(1-\gamma)^2} \nonumber \\
&= \frac{4\gamma \epsilon  }{(1-\gamma)^2}\alpha^2
\end{align}

\end{proof}

\section{Efficiently Solving the Trust-Region Constrained Optimization Problem} \label{sec:cg}

This section describes how to efficiently approximately solve the following constrained optimization problem, which we must solve at each iteration of TRPO:
\begin{align}
\maximize L(\theta) \text{\ \  subject to }  \meankl{}(\thold,\theta) \le \delta.
\end{align}
The method we will describe involves two steps: (1) compute a search direction, using a linear approximation to objective and quadratic approximation to the constraint; and (2) perform a line search in that direction, ensuring that we improve the nonlinear objective while satisfying the nonlinear constraint.

The search direction is computed by approximately solving the equation $Ax=g$, where $A$ is the Fisher information matrix, i.e., the quadratic approximation to the KL divergence constraint:
$\meankl{}(\thold,\theta) \approx \frac{1}{2} (\theta-\thold)^T A (\theta-\thold)$, where $A_{ij} = \frac{\partial}{\partial\theta_i}\frac{\partial}{\partial\theta_j} \meankl{}(\thold,\theta)$.
In large-scale problems, it is prohibitively costly (with respect to computation and memory) to form the full matrix $A$ (or $A^{-1}$).
However, the conjugate gradient algorithm allows us to approximately solve the equation $Ax=b$ without forming this full matrix, when we merely have access to a function that computes matrix-vector products $y \rightarrow Ay$.
\Cref{sec:hvp} describes the most efficient way to compute matrix-vector products with the Fisher information matrix.
For additional exposition on the use of Hessian-vector products for optimizing neural network objectives, see \cite{martens2012training} and \cite{pascanu2013revisiting}.

Having computed the search direction $s \approx A^{-1}g$, we next need to compute the maximal step length $\beta$ such that $\theta + \beta s$ will satisfy the KL divergence constraint.
To do this, let $\delta = \meankl{} \approx \frac{1}{2} (\beta s)^T A (\beta s) = \frac{1}{2}\beta^2 s^T A s$.
From this, we obtain $\beta = \sqrt{2 \delta / s^T A s}$, where $\delta$ is the desired KL divergence.
The term $s^T A s$ can be computed through a single Hessian vector product, and it is also an intermediate result produced by the conjugate gradient algorithm.

Last, we use a line search to ensure improvement of the surrogate objective and satisfaction of the KL divergence constraint, both of which are nonlinear in the parameter vector $\theta$ (and thus depart from the linear and quadratic approximations used to compute the step).
We perform the line search on the objective $L_{\thold}(\theta) - \mathcal{X}[\meankl{}(\thold,\theta) \le \delta]$, where $\mathcal{X}[\dots]$ equals zero when its argument is true and $+\infty$ when it is false.
Starting with the maximal value of the step length $\beta$ computed in the previous paragraph, we shrink $\beta$ exponentially until the objective improves.
Without this line search, the algorithm occasionally computes large steps that cause a catastrophic degradation of performance.

\subsection{Computing the Fisher-Vector Product} \label{sec:hvp}

Here we will describe how to compute the matrix-vector product between the averaged Fisher information matrix and arbitrary vectors. This matrix-vector product enables us to perform the conjugate gradient algorithm.
Suppose that the parameterized policy maps from the input $x$ to ``distribution parameter'' vector $\muth(x)$, which parameterizes the distribution $\pi(u \given x)$.
Now the KL divergence for a given input $x$ can be written as follows:
\begin{align}
\kl{\pi_{\thold}(\cdot \given x)}{\pith(\cdot \given x)} = \klofmu(\muth(x),\muthold)
\end{align}
where $\klofmu$ is the KL divergence between the distributions corresponding to the two mean parameter vectors.
Differentiating $\klofmu$ twice with respect to $\theta$, we obtain
\begin{align}
\frac{\partial \mu_a(x)}{\partial \theta_i}
\frac{\partial \mu_b(x)}{\partial \theta_j}
\klofmu''_{ab}(\muth(x),\muthold)
+
\frac{\partial^2 \mu_a(x)}{\partial \theta_i \partial \theta_j} \klofmu'_{a}(\muth(x),\muthold)
\label{eq:klsecondderiv}
\end{align}
where the primes ($'$) indicate differentiation with respect to the first argument, and there is an implied summation over indices $a,b$.
The second term vanishes, leaving just the first term.
Let $J \defeq \frac{\partial \mu_a(x)}{\partial \theta_i}$ (the Jacobian), then the Fisher information matrix can be written in matrix form as $J^T M J$, where $M=kl''_{ab}(\muth(x),\muthold)$ is the Fisher information matrix of the distribution in terms of the mean parameter $\mu$ (as opposed to the parameter $\theta$).
This has a simple form for most parameterized distributions of interest.

The Fisher-vector product can now be written as a function $y \rightarrow J^T M J y$.
Multiplication by $J^T$ and $J$ can be performed by most automatic differentiation and neural network packages (multiplication by $J^T$ is the well-known backprop operation), and the operation for multiplication by $M$ can be derived for the distribution of interest.
Note that this Fisher-vector product is straightforward to average over a set of datapoints, i.e., inputs $x$ to $\mu$.

One could alternatively use a generic method for calculating Hessian-vector products using reverse mode automatic differentiation (\cite{wright1999numerical}, chapter 8), computing the Hessian of $\meankl{}$ with respect to $\theta$.
This method would be slightly less efficient as it does not exploit the fact that the second derivatives of $\mu(x)$ (i.e., the second term in \Cref{eq:klsecondderiv}) can be ignored, but may be substantially easier to implement.

We have described a procedure for computing the Fisher-vector product $y \rightarrow Ay$, where the Fisher information matrix is averaged over a set of inputs to the function $\mu$.
Computing the Fisher-vector product is typically about as expensive as computing the gradient of an objective that depends on $\mu(x)$ \cite{wright1999numerical}.
Furthermore, we need to compute $k$ of these Fisher-vector products per gradient, where $k$ is the number of iterations of the conjugate gradient algorithm we perform.
We found $k=10$ to be quite effective, and using higher $k$ did not result in faster policy improvement.
Hence, a na\"{i}ve implementation would spend more than $90\%$ of the computational effort on these Fisher-vector products.
However, we can greatly reduce this burden by subsampling the data for the computation of Fisher-vector product.
Since the Fisher information matrix merely acts as a metric, it can be computed on a subset of the data without severely degrading the quality of the final step.
Hence, we can compute it on $10\%$ of the data, and the total cost of Hessian-vector products will be about the same as computing the gradient. With this optimization, the computation of a natural gradient step $A^{-1} g$ does not incur a significant extra computational cost beyond computing the gradient $g$.

\section{Approximating Factored Policies with Neural Networks} \label{sec:nn}

The policy, which is a conditional probability distribution $\pith(a \given s)$, can be parameterized with a neural network.
This neural network maps (deterministically) from the state vector $s$ to a vector $\mu$, which specifies a distribution over action space.
Then we can compute the likelihood $p(a \given \mu)$ and sample $a \sim p(a \given \mu)$.

For our experiments with continuous state and action spaces, we used a Gaussian distribution, where the covariance matrix was diagonal and independent of the state.
A neural network with several fully-connected (dense) layers maps from the input features to the mean of a Gaussian distribution.
A separate set of parameters specifies the log standard deviation of each element.
More concretely, the parameters include a set of weights and biases for the neural network computing the mean, $\lrbrace{W_i, b_i}_{i=1}^L$, and a vector $r$ (log standard deviation) with the same dimension as $a$.
Then, the policy is defined by the normal distribution $\mathcal{N}\left( \mean=\net\left(s; \lrbrace{W_i, b_i}_{i=1}^L \right), \stdev=\exp(r) \right)$.
Here, $\mu = \lrbrack{\mean, \stdev}$.

For the experiments with discrete actions (Atari), we use a factored discrete action space, where each factor is parameterized as a categorical distribution.
That is, the action consists of a tuple $(a_1, a_2, \dots, a_K)$ of integers $a_k \in \lrbrace{1,2,\dots,N_k}$, and each of these components is assumed to have a categorical distribution, which is specified by a vector $\mu_k = [p_1,p_2, \dots, p_{N_k}]$.
Hence, $\mu$ is defined to be the concatenation of the factors' parameters: $\mu = [\mu_1, \mu_2, \dots, \mu_K]$ and has dimension $\dim \mu = \sum_{k=1}^{K} N_k$.
The components of $\mu$ are computed by taking applying a neural network to the input $s$ and then applying the softmax operator to each slice, yielding normalized probabilities for each factor.

\section{Experiment Parameters}
\label{sec:exptparams}

\begin{table}[h!]
\centering
\begin{tabular}{l|ccc}
&   Swimmer & Hopper & Walker\\
\hline
State space dim. &   10  &  12   &  20 \\
Control space dim.  &    2  &   3   &  6   \\
Total num. policy params & 364 & 4806 & 8206 \\
\hline
Sim. steps per iter. & 50K  & 1M &  1M \\
Policy iter. & 200  & 200 & 200 \\
Stepsize ($\meankl{}{}$) & 0.01  & 0.01 & 0.01\\
Hidden layer size & 30 & 50  & 50 \\
Discount ($\gamma$) & 0.99 & 0.99 & 0.99 \\
Vine: rollout length  & 50  & 100 & 100 \\
Vine: rollouts per state & 4 & 4 & 4 \\
Vine: $Q$-values per batch & 500  & 2500 & 2500\\
Vine: num. rollouts for sampling & 16 &  16 & 16 \\
Vine: len. rollouts for sampling & 1000 & 1000 & 1000 \\
Vine: computation time (minutes) & 2 & 14 & 40 \\
SP: num. path & 50 & 1000 & 10000 \\
SP: path len. & 1000 & 1000 & 1000 \\
SP: computation time & 5 & 35 & 100 \\
\end{tabular}
\caption{Parameters for continuous control tasks, vine and single path (SP) algorithms.\label{table:contctrlparams}}
\end{table}

\begin{table}[H]
\centering
\begin{tabular}{l|c}
& All games \\
\hline
Total num. policy params & 33500 \\
\hline
Vine: Sim. steps per iter. & 400K \\
SP: Sim. steps per iter. & 100K \\
Policy iter. & 500  \\
Stepsize ($\meankl{}{}$) & 0.01\\
Discount ($\gamma$) & 0.99 \\
Vine: rollouts per state & $\approx 4$ \\
Vine: computation time & $\approx 30$ hrs \\
SP: computation time & $\approx 30$ hrs \\
\end{tabular}
\caption{Parameters used for Atari domain.\label{table:atariparams}}
\end{table}

\section{Learning Curves for the Atari Domain}
\begin{figure}[h!tp]
\centering
{\includegraphics[width=0.32\textwidth]{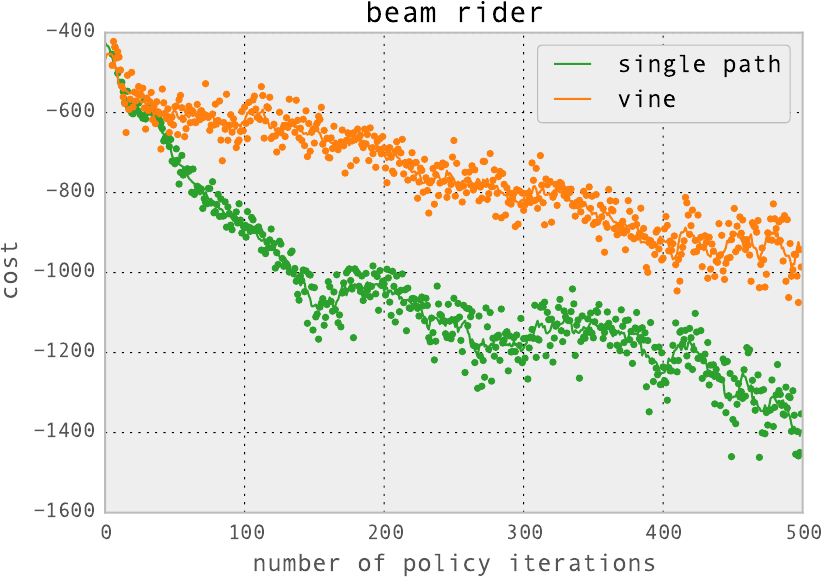}}
\quad
{\includegraphics[width=0.32\textwidth]{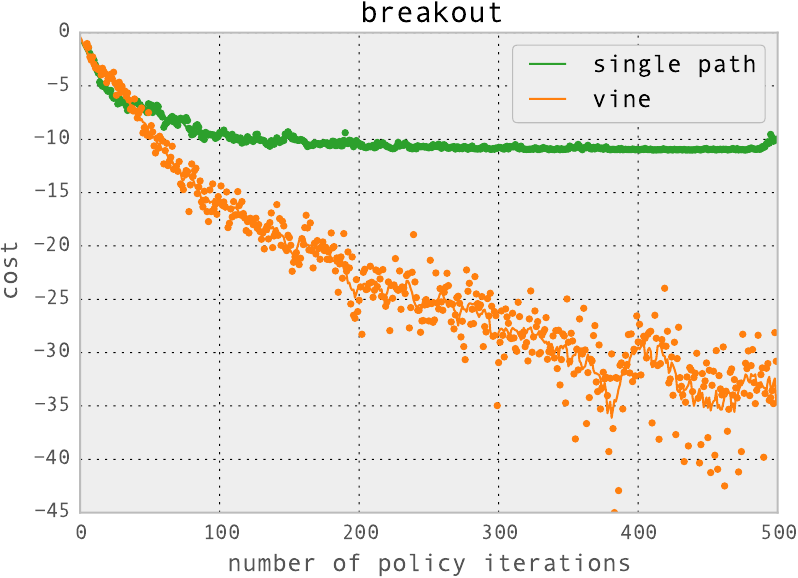}}
{\includegraphics[width=0.32\textwidth]{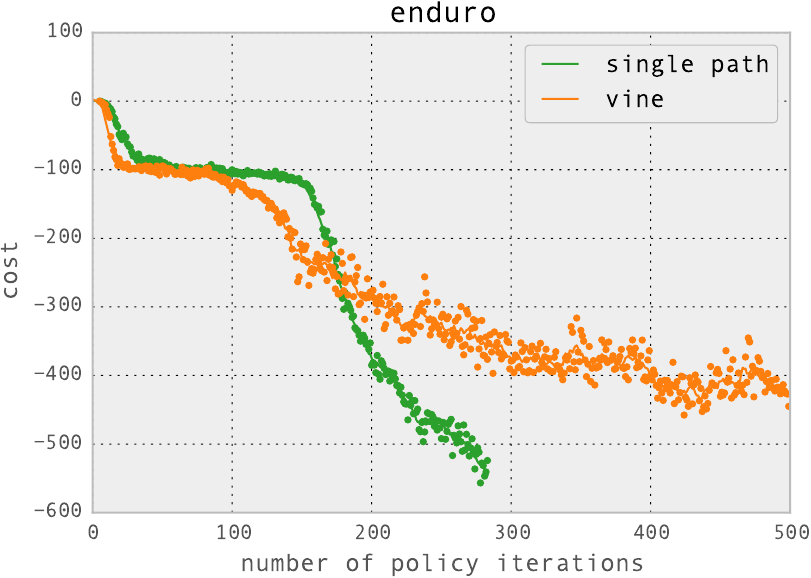}}
{\includegraphics[width=0.32\textwidth]{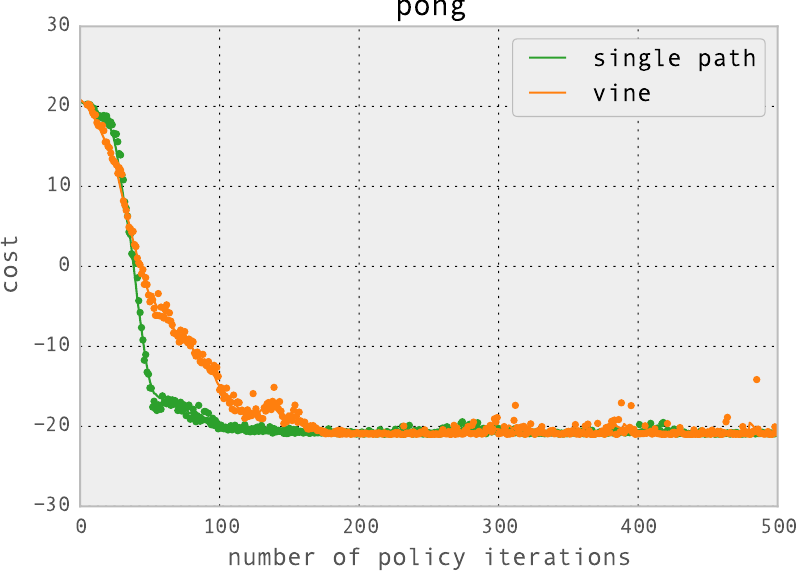}}
{\includegraphics[width=0.32\textwidth]{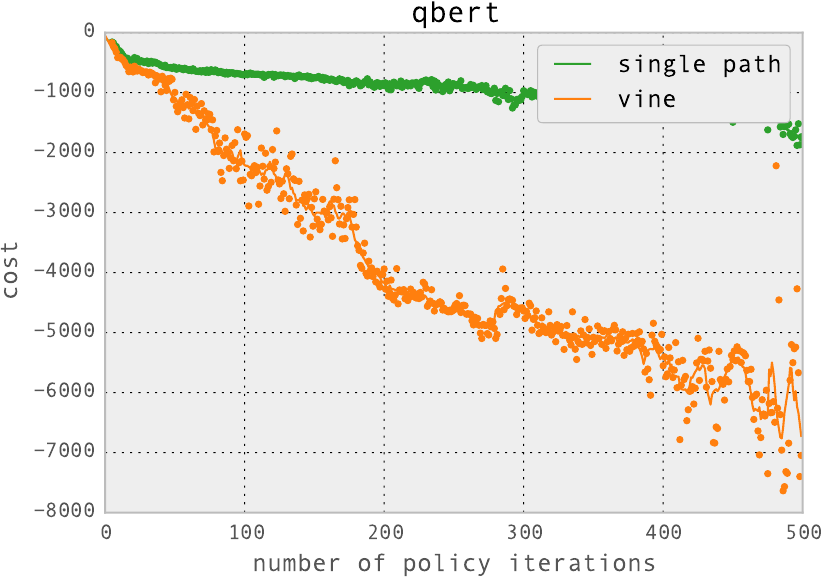}}
{\includegraphics[width=0.32\textwidth]{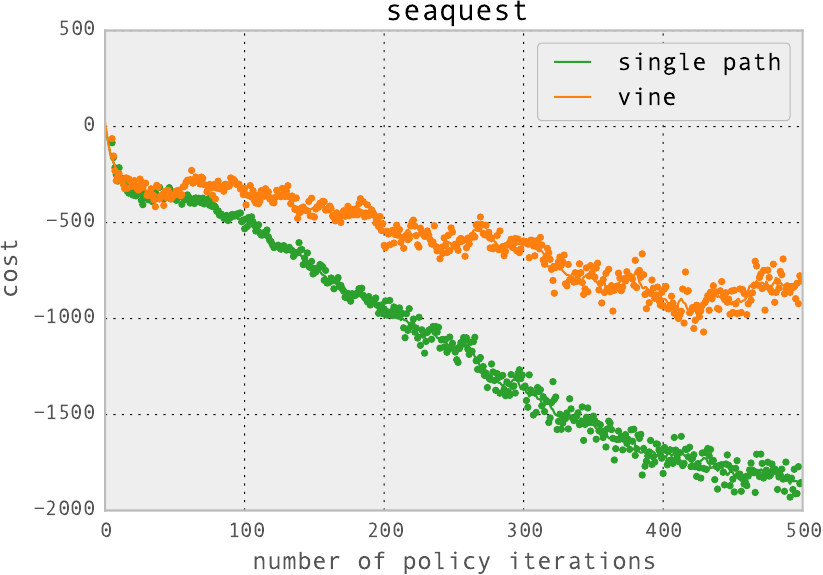}}
{\includegraphics[width=0.32\textwidth]{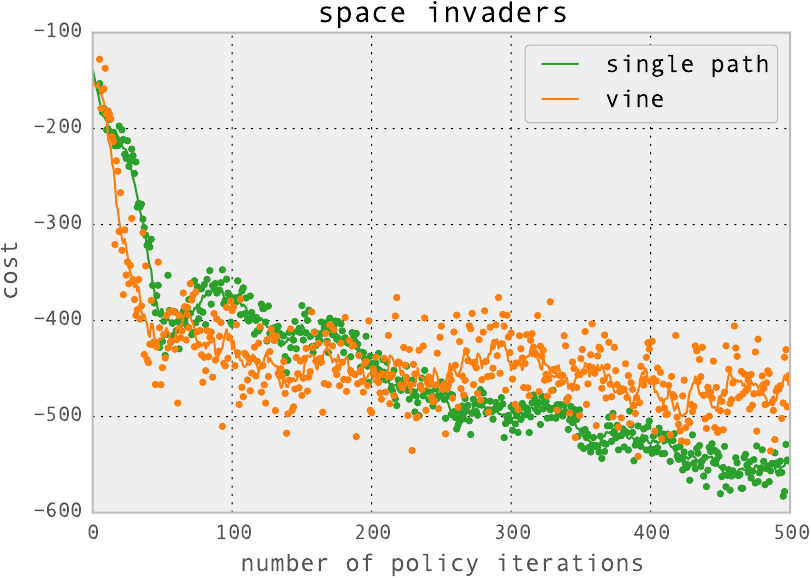}}
\caption{Learning curves for the \Atari{} domain. For historical reasons, the plots show cost = negative reward.}
\label{fig:atari_plots}
\end{figure}

}
\end{document}